\documentclass[pmlr,twocolumn,a4paper,10pt]{jmlr}
\usepackage{isipta2023}
\usepackage[american]{babel} 
\usepackage{tikz} 
\usetikzlibrary{patterns}

\newcommand{\pr}[1]{\mathbb{P}\!\left( #1 \right)}
\newcommand{\indep}{\rotatebox[origin=c]{90}{$\models$}}

\jmlrworkshop{ISIPTA 2023}

\title{Markov Conditions and Factorization in Logical Credal Networks}

\author{
  \Name{Fabio Gagliardi Cozman} \\
  \addr Universidade de S\~ao Paulo, Brazil 
}

\begin{document}
\maketitle

\begin{abstract}
We examine the recently proposed language of {\em Logical Credal Networks},
in particular investigating the consequences of various Markov conditions.
We introduce the notion of structure for a Logical Credal Network and show
that a structure without directed cycles leads to a well-known factorization
result. For networks with directed cycles, we analyze the differences
between Markov conditions, factorization results, and specification 
requirements.  
\end{abstract}

\section{Introduction}\label{sec:intro}
 
This paper examines {\em Logical Credal Networks}, a formalism
recently introduced by \citet{Qian2022}  to combine
logical sentences, probabilities and independence relations. 
The have proposed interesting ideas and evaluated the 
formalism in practical scenarios with positive results.

The central element of a Logical Credal Network (LCN) 
is a collection of   constraints
over probabilities. Independence relations are then extracted mostly 
from the logical content of those inequalities. This scheme 
differs from previous proposals that extract independence relations 
from explicitly
specified graphs~\cite{Andersen96,CozmanUAI2009,Rocha2005}.
Several probabilistic logics have also adopted explicit syntax
for independence relations even when graphs are not employed
\cite{Bacchus90,Doder2017,Halpern2003}.

While Logical Credal Networks have points in common with
  existing formalisms, they do
have   novel features that deserve attention.
For one thing, they resort to directed graphs that may
contain directed cycles. Also they are endowed with a 
sensible Markov condition that is distinct from previous ones.
Little is known about the consequences of these features, 
and how they interact with the syntactic conventions that
turn logical formulas into edges in graphs.
In particular, it seems that no study has focused on
the consequences of Markov conditions on factorization results;
that is, how such conditions affect the factors that constitute 
probability distributions.

In this paper we present a first investigation towards a deeper understanding
of Logical Credal Networks, looking at their specification, their 
Markov conditions, their factorization properties. 
We introduce the notion of ``structure'' for a LCN.
We then show that the local Markov condition proposed by \citet{Qian2022}
collapses to the usual local Markov condition applied to chain graphs 
when the structure has no directed cycles.
We analyze the behavior of the former Markov condition in the presence
of directed cycles, in particular examining factorization properties and
discussing the semantics of the resulting language. 
To conclude, we propose a novel semantics for 
LCNs and examine factorization results for their
associated probability distributions. 

\section{Graphs and Markov Conditions}\label{section:Graphs}

In this section we present the necessary concepts related to
graphs and graph-theoretical probabilistic models (Bayesian
networks, Markov networks, and chain graphs). Definitions and
notation vary across the huge literature on these topics;
we rely here on three sources. 
We use definitions by \citet{Qian2022} and by \citet{Spirtes95}
in their work on LCNs and on directed graphs respectively; we also
use standard results from the textbook by \citet{Cowell99}.  

A {\em graph} is a triple $(\mathcal{V},\mathcal{E}_D,\mathcal{E}_U)$, 
where $\mathcal{V}$ is a set of nodes,
and both $\mathcal{E}_D$ and $\mathcal{E}_U$ are sets of edges.
A node is always labeled with the name of a random variable; in fact, 
we do not distinguish between a node and the corresponding random variable.
The elements of $\mathcal{E}_D$ are {\em directed} edges.
A directed edge is an ordered pair of distinct nodes, and is denoted 
by $A \rightarrow B$.
The elements of $\mathcal{E}_U$ are {\em undirected} edges.
An undirected edge is a pair of distinct nodes, and is denoted by $A \sim B$;
note that nodes are not ordered in an undirected edge, so
there is no difference between $A \sim B$ and $B \sim A$. 
Note that $\mathcal{E}_D$ and $\mathcal{E}_U$ are sets, so there are no multiple copies
of elements in them (for instance, there are no multiple undirected 
edges between two nodes). Note also that there is no loop from a 
node to itself.  
 
If there is a directed edge from $A$ to $B$, 
the edge is said to be {\em from} $A$ to $B$,
and then $A$ is a {\em parent} of $B$ and $B$ is a {\em child} of $A$.
The parents of $A$ are denoted by $\mbox{pa}(A)$.
If there are directed edges $A \rightarrow B$ and $B \rightarrow A$ between
$A$ and $B$, we say there is {\em bi-directed} edge
between $A$ and $B$ and write $A \rightleftarrows B$. 
If $A \sim B$, then both nodes are said to be {\em neighbors}.
The neighbors of $A$ are denoted by $\mbox{ne}(A)$.
The {\em boundary} of a node $A$, denoted by $\mbox{bd}(A)$,
is the set $\mbox{pa}(A) \cup \mbox{ne}(A)$.
The boundary of a set $\mathcal{B}$ of nodes is $\mbox{bd}(\mathcal{B}) = \cup_{A \in \mathcal{B}} \mbox{bd}(A) \backslash \mathcal{B}$.
If we have a set $\mathcal{B}$ of nodes such that, for all $A \in \mathcal{B}$,  the boundary of $A$ contained in $\mathcal{B}$, then $\mathcal{B}$ is an {\em ancestral set}. 

A {\em path} from $A$ to $B$ is a sequence of edges, the first one between $A$ and some node $C_1$, then from $C_1$ to $C_2$ and so on, until an edge from $C_k$ to $B$, where all nodes other than $A$ and $B$
are distinct, and such that for each pair $(D_1,D_2)$ of consecutive 
nodes in the path we have either $D_1 \rightarrow D_2$ or $D_1 \sim D_2$ 
but never $D_2 \rightarrow D_1$.
If $A$ and $B$ are identical, the path is a {\em cycle}.
If there is at least one directed edge in a path, the path is a {\em directed path}; if that path is a cycle, then it is a {\em directed cycle}. 
If a path is not directed, then it is {\em undirected} 
(hence all edges in the path are undirected ones). 
A {\em directed}/{\em undirected} graph is a graph that only contains
directed/undirected edges.
A graph without directed cycles is a {\em chain graph}.
Figure~\ref{figure:GraphExamples} depicts a number of graphs.

\begin{figure}[t]
\begin{tikzpicture}[scale=0.9]
\node[draw,rectangle,rounded corners,fill=yellow] (a) at (1,2) {$A$};
\node[draw,rectangle,rounded corners,fill=yellow] (b) at (2,2) {$B$};
\node[draw,rectangle,rounded corners,fill=yellow] (c) at (1,1) {$C$};
\node[draw,rectangle,rounded corners,fill=yellow] (d) at (2,1) {$D$};
\draw[->,>=latex,thick] (a)--(b);
\draw[->,>=latex,thick] (c)--(d);
\draw[->,>=latex,thick] (b) -- (d); 
\node at (1.5,0.3) {\small (a)};
\end{tikzpicture}
\hspace*{1mm}
\begin{tikzpicture}[scale=0.9]
\node[draw,rectangle,rounded corners,fill=yellow] (a) at (1,2) {$A$};
\node[draw,rectangle,rounded corners,fill=yellow] (b) at (2,2) {$B$};
\node[draw,rectangle,rounded corners,fill=yellow] (c) at (1,1) {$C$};
\node[draw,rectangle,rounded corners,fill=yellow] (d) at (2,1) {$D$};
\draw[->,>=latex,thick] (a) to[out=10,in=170] (b);
\draw[->,>=latex,thick] (b) to[out=-170,in=-10] (a);
\draw[->,>=latex,thick] (c) to[out=10,in=170] (d);
\draw[->,>=latex,thick] (d) to[out=-170,in=-10] (c); 
\draw[->,>=latex,thick] (b) to[out=-80,in=80] (d);
\draw[->,>=latex,thick] (d) to[out=100,in=-100] (b);
\node at (1.5,0.3) {\small (b)};
\end{tikzpicture}
\hspace*{1mm}
\begin{tikzpicture}[scale=0.9]
\node[draw,rectangle,rounded corners,fill=yellow] (a) at (1,2) {$A$};
\node[draw,rectangle,rounded corners,fill=yellow] (b) at (2,2) {$B$};
\node[draw,rectangle,rounded corners,fill=yellow] (c) at (1,1) {$C$};
\node[draw,rectangle,rounded corners,fill=yellow] (d) at (2,1) {$D$};
\draw[thick] (a)--(b);
\draw[thick] (c)--(d);
\draw[thick] (b) -- (d);
\node at (1.5,0.3) {\small (c)};
\end{tikzpicture}
\hspace*{1mm}
\begin{tikzpicture}[scale=0.9]
\node[draw,rectangle,rounded corners,fill=yellow] (a) at (1,2) {$A$};
\node[draw,rectangle,rounded corners,fill=yellow] (b) at (2,2) {$B$};
\node[draw,rectangle,rounded corners,fill=yellow] (c) at (1,1) {$C$};
\node[draw,rectangle,rounded corners,fill=yellow] (d) at (2,1) {$D$};
\draw[->,>=latex,thick] (a)--(b);
\draw[->,>=latex,thick] (c)--(d);
\draw[->,>=latex,thick] (b) to[out=-80,in=80] (d);
\draw[->,>=latex,thick] (d) to[out=100,in=-100] (b);
\node at (1.5,0.3) {\small (d)};
\end{tikzpicture}
\hspace*{1mm}
\begin{tikzpicture}[scale=0.9]
\node[draw,rectangle,rounded corners,fill=yellow] (a) at (1,2) {$A$};
\node[draw,rectangle,rounded corners,fill=yellow] (b) at (2,2) {$B$};
\node[draw,rectangle,rounded corners,fill=yellow] (c) at (1,1) {$C$};
\node[draw,rectangle,rounded corners,fill=yellow] (d) at (2,1) {$D$};
\draw[->,>=latex,thick] (a)--(b);
\draw[->,>=latex,thick] (c)--(d);
\draw[thick] (b) -- (d); 
\node at (1.5,0.3) {\small (e)};
\end{tikzpicture}
\vspace*{-5ex}
\caption{Graphs (directed/directed/undirected/directed/ chain).
We have $\mbox{pa}(B)=\{A\}$ in Figures \ref{figure:GraphExamples}.a and \ref{figure:GraphExamples}.e, $\mbox{pa}(B)=\{A,D\}$ in Figures \ref{figure:GraphExamples}.b and \ref{figure:GraphExamples}.d, 
and $\mbox{pa}(B)=\emptyset$ in Figure \ref{figure:GraphExamples}.c.}
\label{figure:GraphExamples}
\end{figure}
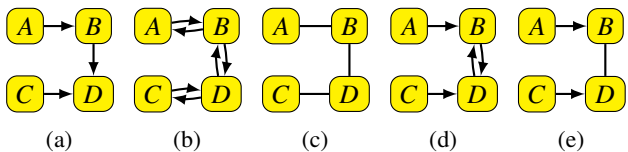

\begin{figure}[t]
\begin{tikzpicture}[scale=0.9]
\node[draw,rectangle,rounded corners,fill=yellow] (a) at (1,2) {$A$};
\node[draw,rectangle,rounded corners,fill=yellow] (b) at (2,2) {$B$};
\node[draw,rectangle,rounded corners,fill=yellow] (c) at (1,1) {$C$};
\node[draw,rectangle,rounded corners,fill=yellow] (d) at (2,1) {$D$};
\draw[thick] (a)--(b);
\draw[thick] (c)--(d);
\draw[thick] (b) -- (d); 
\draw[thick] (c) -- (b);
\node at (1.5,0.3) {\small (a)};
\end{tikzpicture}
\hspace*{1mm}
\begin{tikzpicture}[scale=0.9]
\node[draw,rectangle,rounded corners,fill=yellow] (a) at (1,2) {$A$};
\node[draw,rectangle,rounded corners,fill=yellow] (b) at (2,2) {$B$};
\node[draw,rectangle,rounded corners,fill=yellow] (c) at (1,1) {$C$};
\node[draw,rectangle,rounded corners,fill=yellow] (d) at (2,1) {$D$};
\draw[thick] (a) -- (b); 
\draw[thick] (c) -- (d);
\draw[thick] (b) -- (d);
\draw[thick] (a) -- (d);
\draw[thick] (c) -- (b);
\node at (1.5,0.3) {\small (b)};
\end{tikzpicture}
\hspace*{1mm}
\begin{tikzpicture}[scale=0.9]
\node[draw,rectangle,rounded corners,fill=yellow] (a) at (1,2) {$A$};
\node[draw,rectangle,rounded corners,fill=yellow] (b) at (2,2) {$B$};
\node[draw,rectangle,rounded corners,fill=yellow] (c) at (1,1) {$C$};
\node[draw,rectangle,rounded corners,fill=yellow] (d) at (2,1) {$D$};
\draw[thick] (a)--(b);
\draw[thick] (c)--(d);
\draw[thick] (b) -- (d);
\node at (1.5,0.3) {\small (c)};
\end{tikzpicture}
\hspace*{1mm}
\begin{tikzpicture}[scale=0.9]
\node[draw,rectangle,rounded corners,fill=yellow] (a) at (1,2) {$A$};
\node[draw,rectangle,rounded corners,fill=yellow] (b) at (2,2) {$B$};
\node[draw,rectangle,rounded corners,fill=yellow] (c) at (1,1) {$C$};
\node[draw,rectangle,rounded corners,fill=yellow] (d) at (2,1) {$D$};
\draw[thick] (a) -- (b); 
\draw[thick] (c) -- (d);
\draw[thick] (b) -- (d);
\draw[thick] (a) -- (d);
\draw[thick] (c) -- (b);
\node at (1.5,0.3) {\small (d)};
\end{tikzpicture}
\hspace*{1mm}
\begin{tikzpicture}[scale=0.9]
\node[draw,rectangle,rounded corners,fill=yellow] (a) at (1,2) {$A$};
\node[draw,rectangle,rounded corners,fill=yellow] (b) at (2,2) {$B$};
\node[draw,rectangle,rounded corners,fill=yellow] (c) at (1,1) {$C$};
\node[draw,rectangle,rounded corners,fill=yellow] (d) at (2,1) {$D$};
\draw[thick] (a)--(b);
\draw[thick] (c)--(d);
\draw[thick] (b) -- (d); 
\draw[thick] (a) -- (c);
\node at (1.5,0.3) {\small (e)};
\end{tikzpicture}
\vspace*{-5ex}
\caption{The moral graphs of the graphs in Figure~\ref{figure:GraphExamples}.}
\label{figure:MoralGraphs}
\end{figure}
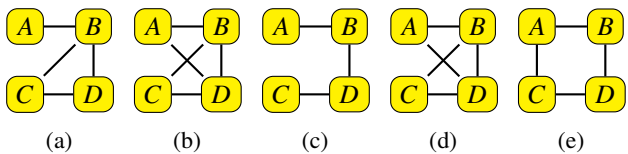

If there is a directed path from $A$ to $B$, then
$A$ is an {\em ancestor} of $B$ and $B$ is a {\em descendant} 
of $A$.
For instance, in Figure \ref{figure:GraphExamples}.a, 
$A$ is the ancestor of $B$ and $D$ is the descendant of $B$;
in Figure \ref{figure:GraphExamples}.c, there are no ancestors
nor descendants of $B$; 
in Figure \ref{figure:GraphExamples}.e, 
$A$ is the ancestor of $B$, and there are no descendants of $B$.
As a digression, note that \citet{Cowell99} define ``ancestor''
and ``descendant'' somewhat differently, by asking that there is
a path from $A$ to $B$ but not from $B$ to $A$; this definition
is equivalent to the previous one for graphs without directed
cycles, but it is different otherwise (for instance, 
in Figure \ref{figure:GraphExamples}.b the node $B$ has
descendants $\{A,C,D\}$ in the previous definition but no
descendants in the sense of \citet{Cowell99}). 
We stick to our former definition, a popular one \cite{Koller2009}
that seems appropriate in the presence of directed cycles \cite{Spirtes95}.

We will need the following  concepts:
\begin{itemize} 
\item Suppose we take graph $\mathcal{G}$ and remove its directed edges to obtain
an auxiliary undirected graph $\mathcal{G}'$. 
A set of nodes $\mathcal{B}$ is a 
{\em chain multi-component} of $\mathcal{G}$ iff 
every pair of nodes in $\mathcal{B}$ is connected by a path
in $\mathcal{G}'$.
And $\mathcal{B}$ is a {\em chain component} iff it is {\em either}
a chain multi-component {\em or}
a single node that does not belong to any chain multi-component.
\item Suppose we take graph $\mathcal{G}$ and add undirected edges
between all pairs of nodes that have a children in a common chain component
of $\mathcal{G}$ and that are not already joined in $\mathcal{G}$.
Suppose we then take the resulting graph and transform every directed edge
into an undirected edge (if $A \leftrightarrows B$, then both
transformed undirected edges collapse into $A \sim B$). 
The final result is the {\em moral graph} of $\mathcal{G}$, denoted
by $\mathcal{G}^m$.
\item Suppose we take a graph $\mathcal{G}$ and a triple $(\mathcal{N}_1,\mathcal{N}_2,\mathcal{N}_3)$
of disjoint subsets of nodes, 
and we build the moral graph of the smallest ancestral set containing the nodes in $\mathcal{N}_1 \cup \mathcal{N}_2 \cup \mathcal{N}_3$.
The resulting graph is denoted by 
$\mathcal{G}^{ma}(\mathcal{N}_1,\mathcal{N}_2,\mathcal{N}_3)$.
\end{itemize}

Figure \ref{figure:MoralGraphs} depicts the moral graphs 
that correspond respectively to the five graphs in 
Figure \ref{figure:GraphExamples}.

There are several formalisms that employ graphs to represent
stochastic independence (and dependence) relations among the
random variables associated with nodes. In this paper we
focus only on discrete random variables, so the concept of
stochastic independence is quite simple:
random variables $X$ and $Y$ are (conditionally) independent 
given random variables $Z$ iff 
$\pr{X=x,Y=y|Z=z} = \pr{X=x|Z=z}\pr{Y=y|Z=z}$ for every possible
$x$ and $y$ and every $z$ such that $\pr{Z=z}>0$. 
In case $Z$ is absent, we have independence of $X$ and $Y$ iff
$\pr{X=x,Y=y}=\pr{X=x}\pr{Y=y}$ for every possible $x$ and $y$. 

A {\em Markov condition} explains how to extract independent relations
from a graph; there are many such conditions in the literature
\cite{Cowell99}.

Consider first an undirected graph $\mathcal{G}$ with set of nodes
$\mathcal{N}$.
The {\em local Markov condition}
states that a node $A$ is independent of all nodes 
in  $\mathcal{N}$ other than $A$ itself and $A$'s neighbors, $\mbox{ne}(A)$,  
given $\mbox{ne}(A)$. 
The {\em global Markov condition}
states that, given any triple $(\mathcal{N}_1,\mathcal{N}_2,\mathcal{N}_3)$
of disjoint subsets of $\mathcal{N}$, such that $\mathcal{N}_2$ separates
$\mathcal{N}_1$ and $\mathcal{N}_3$, then nodes $\mathcal{N}_1$ and
$\mathcal{N}_3$ are independent given nodes $\mathcal{N}_2$.\footnote{In
an undirected graph, a set of nodes separates two other sets iff,
by deleting the separating nodes, we have no connecting path between
a node in one set and a node in the other set.}
If a probability distribution over all random variables in $\mathcal{G}$ 
is everywhere larger than zero, then both conditions are equivalent
and they are equivalent to a {\em factorization property}: 
for each configuration of variables $X=x$, where $X$ denotes the
random variables in $\mathcal{G}$, we have 
$\pr{X=x} = \prod_{c \in \mathcal{C}} \phi_c(x_c)$,
where $\mathcal{C}$ is the set of cliques of $\mathcal{G}$,
each $\phi_c$ is a function over the random variables in clique
$c$, and $x_c$ is the projection of $x$ over the random variables
in clique $c$.\footnote{A clique is a maximal set of nodes such
that each pair of nodes in the set is joined.}

Now consider an acyclic 
directed graph $\mathcal{G}$ with set of nodes $\mathcal{N}$.
The {\em local Markov condition} 
states that a node $A$ is independent, given $A$'s parents $\mbox{pa}(A)$,
of all its non-descendants non-parents except $A$ itself. 
The factorization produced by the local Markov condition is
\begin{equation}\label{equation:BayesNet}
\pr{X=x} = \prod_{N \in \mathcal{N}} \pr{N=x_N|\mbox{pa}(N)=x_{\mbox{pa}(N)}},
\end{equation}
where $x_N$ is the value of $N$ in $x$ and $x_{\mbox{pa}(N)}$ is the projection
of $x$ over the parents of $N$.

Finally, consider a chain graph $\mathcal{G}$ with set of nodes $\mathcal{N}$.
The {\em local Markov condition} for chain graphs is:
\begin{definition}[LMC(C)]\label{definition:LMCchaingraphs}
A node $A$ is independent, given its parents,
of all nodes that are not $A$ itself nor descendants nor boundary nodes of $A$.
\end{definition}
The {\em global Markov condition} for chain graphs is significantly more complicated:
\begin{definition}[GMC(C)]\label{definition:SpirtesGMC}
Given any triple $(\mathcal{N}_1,\mathcal{N}_2,\mathcal{N}_3)$
of disjoint subsets of $\mathcal{N}$,
if $\mathcal{N}_2$ separates 
$\mathcal{N}_1$ and $\mathcal{N}_2$ in the graph 
$\mathcal{G}^{ma}(\mathcal{N}_1,\mathcal{N}_2,\mathcal{N}_3)$, 
then nodes $\mathcal{N}_1$ 
and $\mathcal{N}_3$ are independent given nodes $\mathcal{N}_2$.
\end{definition}
Again, if a probability distribution over all random variables in $\mathcal{G}$ 
is everywhere larger than zero, then both Markov conditions are equivalent
and they are equivalent to a {\em factorization property} as follows. 
Take the chain components $T_1, \dots, T_n$ ordered so that nodes in $T_i$ can only be at the end of directed edges starting from chain compoents before $T_i$; this is always possible in a chain graph. 
Then the factorization has the form
$\pr{X=x} = \prod_{i=1}^n \pr{\mathcal{N}_i=x_{\mathcal{N}_i}\mid\mbox{bd}(\mathcal{N}_i)=x_{\mbox{\scriptsize bd}(\mathcal{N}_i)} }$
where $\mathcal{N}_i$ is the set of nodes in the $i$th chain component;
$x_{\mathcal{N}_i}$ and $x_{\mbox{\scriptsize bd}(\mathcal{N}_i)}$
are respectively the projection of $x$
over $\mathcal{N}_i$ and 
$\mbox{bd}(\mathcal{N}_i)$.
Moreover, each
factor in the product itself factorizes accordingly to an undirected graph that depends on the corresponding chain 
component~\cite{Cowell99}. 
More precisely, for each chain component $T_i$,
build an undirected graph consisting of the nodes in $\mathcal{N}_i$ and $\mbox{bd}(\mathcal{N}_i)$ with all edges between these nodes in $\mathcal{G}$ turned into undirected 
edges in this new graph, and with new undirected edges connecting each pair of nodes in  $\mbox{bd}(\mathcal{N}_i)$ that were not joined already; then each
$\pr{\mathcal{N}_i\mid\mbox{bd}(\mathcal{N}_i)}$ equals the ratio
$\phi_i(\mathcal{N}_i,\mbox{bd}(\mathcal{N}_i))/\phi_i(\mbox{bd}(\mathcal{N}_i)))$ for positive function  $\phi_i$, where $\phi_i(\mbox{bd}(\mathcal{N}_i))) = \sum \phi_i(\mathcal{N}_i,\mbox{bd}(\mathcal{N}_i))$ with the sum extending over all configurations of $\mathcal{N}_i$.

\section{Logical Credal Networks}

A Logical Credal Network (LCN) consists of a set of propositions $\mathcal{N}$ and two sets of constraints $\mathcal{T}_U$ and $\mathcal{T}_D$.
The set $\mathcal{N}$ is finite with  propositions $A_1, \dots, A_n$.
Each proposition $A_i$ is associated with a random variable $X_i$ that is an indicator variable: if $A_i$ holds in an interpretation of the propositions then $X_i=1$; otherwise, $X_i=0$. 
From now on we simply use the same symbol for a proposition and its corresponding indicator random variable. 
Each constraint in $\mathcal{T}_U$ and in $\mathcal{T}_D$ is of the form
\[
\alpha \leq \pr{\phi|\varphi} \leq \beta,
\]
where each $\phi$ and each $\varphi$ is a formula.
In this paper, formulas are propositional (with propositions in $\mathcal{N}$ and connectives such as negation, disjunction, conjunction).
The definition of LCNs by \citet{Qian2022} allows for relational structures and first-order formulas; however, their semantics is obtained by grounding on finite domains, so we can focus on a propositional language for our purposes here.

Note that $\varphi$ can be a tautology $\top$, in which case we can just write the ``unconditional'' probability $\pr{\phi}$.
One can obviously use simple variants of constraints, such as $\pr{\phi|\varphi} = \beta$ or $\pr{\phi|\varphi} \geq \alpha$ or $\pr{\phi} \leq \alpha$, whenever needed. 

The semantics of a LCN is given by a translation from the LCN to a directed graph where each   proposition/random variable is a node (we often refer to
them as {\em proposition-nodes}).
Each constraint is then processed as follows.
First, a node labeled with formula $\phi$ is added and, in case $\varphi$ is not $\top$, another node labeled with $\varphi$ is added (we often refer  to them as
{\em formula-nodes}), with a directed edge from $\varphi$ to $\phi$. 
Then an edge is added from each proposition in $\varphi$ to node $\varphi$ in case the latter is in the graph, and an edge is added from node $\phi$ to each proposition in $\phi$.\footnote{We note that the original presentation of LCNs is a bit different from what we just described, as there are no edges added for a constraint in $\mathcal{T}_D$ for which $\varphi$ is $\top$. But this does not make any difference in the results and simplifies a bit the discussion.}
Finally, in case the constraint is in $\mathcal{T}_U$, an edge is added from each proposition in $\phi$ to node $\phi$. 
We do not distinguish between two logically equivalent formulas (the original proposal
by \citet{Qian2022} focused only on syntactic operations).

The graph just described is referred to as the 
{\em dependency graph} of the LCN.
Note that, when a formula is just a single proposition, we do not need to 
present it explicitly in the dependency graph; we can just connect
edges from and to the corresponding proposition-node. 
As shown in the next example, in our drawings 
formulas appear inside dashed rectangles.

\begin{example}\label{example:Smokers}
Consider the following LCN, based on the Smokers and Friends example
by \citet{Qian2022}. We have propositions 
$C_i$, $F_i$, $S_i$ for $i\in\{1,2,3\}$.
All constraints belong to $\mathcal{T}_U$ (that is, $\mathcal{T}_D$ is empty),
with $i,j \in \{1,2,3\}$:
\[
\begin{array}{ll}
0.5 \leq \pr{F_i | F_j \wedge F_k} \leq 1, & i \neq j, i \neq k, j \neq k; \\
0 \leq \pr{S_i \vee S_j | F_i} \leq 0.2, & i \neq j; \\
0.03 \leq \pr{C_i|S_i} \leq 0.04; & \\
0 \leq \pr{C_i | \neg S_i} \leq 0.01.
\end{array}
\]
The structure of the LCN is depicted in Figure \ref{figure:SmokersExample}.
Note that there are several directed cycles in this dependency graph.
\end{example}

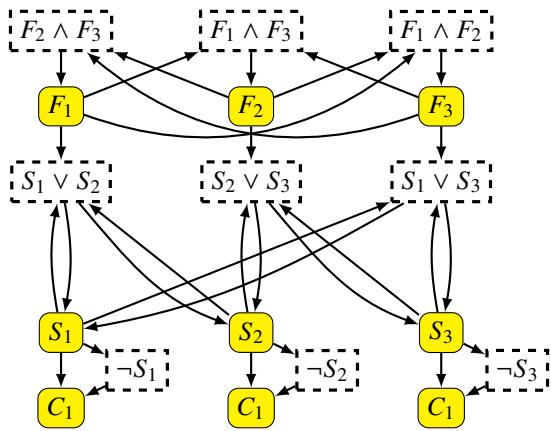
\begin{figure}[t]
\centering
\begin{tikzpicture}
\node[draw,rectangle,very thick,dashed] (f2f3) at (1,5) {$F_2 \wedge F_3$};
\node[draw,rectangle,very thick,dashed] (f1f3) at (3.5,5) {$F_1 \wedge F_3$};
\node[draw,rectangle,very thick,dashed] (f1f2) at (6,5) {$F_1 \wedge F_2$};
\node[draw,rectangle,rounded corners,fill=yellow] (f1) at (1,4) {$F_1$};
\node[draw,rectangle,rounded corners,fill=yellow] (f2) at (3.5,4) {$F_2$};
\node[draw,rectangle,rounded corners,fill=yellow] (f3) at (6,4) {$F_3$};
\node[draw,rectangle,very thick,dashed] (s1s2) at (1,3) {$S_1 \vee S_2$};
\node[draw,rectangle,very thick,dashed] (s2s3) at (3.5,3) {$S_2 \vee S_3$};
\node[draw,rectangle,very thick,dashed] (s1s3) at (6,3) {$S_1 \vee S_3$};
\node[draw,rectangle,rounded corners,fill=yellow] (s1) at (1,1) {$S_1$};
\node[draw,rectangle,rounded corners,fill=yellow] (s2) at (3.5,1) {$S_2$};
\node[draw,rectangle,rounded corners,fill=yellow] (s3) at (6,1) {$S_3$};
\node[draw,rectangle,very thick,dashed] (ns1) at (2,0.5) {$\neg S_1$};
\node[draw,rectangle,very thick,dashed] (ns2) at (4.5,0.5) {$\neg S_2$};
\node[draw,rectangle,very thick,dashed] (ns3) at (7,0.5) {$\neg S_3$};
\node[draw,rectangle,rounded corners,fill=yellow] (c1) at (1,0) {$C_1$};
\node[draw,rectangle,rounded corners,fill=yellow] (c2) at (3.5,0) {$C_1$};
\node[draw,rectangle,rounded corners,fill=yellow] (c3) at (6,0) {$C_1$};
\draw[->,>=latex,thick] (f2f3)--(f1);
\draw[->,>=latex,thick] (f1f3)--(f2);
\draw[->,>=latex,thick] (f1f2)--(f3);
\draw[->,>=latex,thick] (f1)--(f1f3);
\draw[->,>=latex,thick] (f2)--(f2f3);
\draw[->,>=latex,thick] (f2)--(f1f2);
\draw[->,>=latex,thick] (f3)--(f1f3);
\draw[->,>=latex,thick] (f3) to[out=-160,in=-40] (f2f3);
\draw[->,>=latex,thick] (f1) to[out=-20,in=-140] (f1f2);
\draw[->,>=latex,thick] (f1)--(s1s2);
\draw[->,>=latex,thick] (f2)--(s2s3);
\draw[->,>=latex,thick] (f3)--(s1s3);
\draw[->,>=latex,thick] (s1)--(c1); 
\draw[->,>=latex,thick] (s2)--(c2); 
\draw[->,>=latex,thick] (s3)--(c3); 
\draw[->,>=latex,thick] (s1s2) to[out=-80,in=80] (s1);
\draw[->,>=latex,thick] (s1) to[out=100,in=-100] (s1s2);
\draw[->,>=latex,thick] (s2s3) to[out=-80,in=80] (s2);
\draw[->,>=latex,thick] (s2) to[out=100,in=-100] (s2s3);
\draw[->,>=latex,thick] (s1s3) to[out=-80,in=80] (s3);
\draw[->,>=latex,thick] (s3) to[out=100,in=-100] (s1s3);
\draw[->,>=latex,thick] (s1s2) to[out=-50,in=160] (s2);
\draw[->,>=latex,thick] (s2) -- (s1s2);
\draw[->,>=latex,thick] (s2s3) to[out=-50,in=160] (s3);
\draw[->,>=latex,thick] (s3) -- (s2s3);
\draw[->,>=latex,thick] (s1s3) to[out=-150,in=10] (s1);
\draw[->,>=latex,thick] (s1) -- (s1s3);
\draw[->,>=latex,thick] (s1) -- (ns1);
\draw[->,>=latex,thick] (s2) -- (ns2);
\draw[->,>=latex,thick] (s3) -- (ns3);
\draw[->,>=latex,thick] (ns1) -- (c1);
\draw[->,>=latex,thick] (ns2) -- (c2);
\draw[->,>=latex,thick] (ns3) -- (c3);
\end{tikzpicture}
\caption{The dependency graph of the LCN in Example~\ref{example:Smokers}.}
\label{figure:SmokersExample}
\end{figure}

\citet{Qian2022} then define:
\begin{definition}\label{definition:ParentLCN}
The {\em lcn-parents} of a proposition $A$, denoted by $\mbox{\rm lcn-pa}(A)$, 
are the propositions such
that there exists a directed path in the dependency graph from
each of them to $A$ in which all intermediate nodes are formulas.
\end{definition}
\begin{definition}\label{definition:DescendantLCN}
The {\em lcn-descendants} of a proposition $A$, denoted by $\mbox{\rm lcn-de}(A)$,
are the propositions such
that there exists a directed path in the dependency graph from $A$ to each 
of them in which no intermediate node is a parent of~$A$. 
\end{definition}
The connections between these concepts and the definitions
of parent and descendant in Section~\ref{section:Graphs}
will be clear in the next section. 

In any case, using these definitions \citet{Qian2022} proposed 
a Markov condition: 
\begin{definition}[LMC(LCN)]
\label{definition:LCNcondition}
A node $A$ is independent, given its lcn-parents,
of all nodes that are not $A$ itself nor lcn-descendants of $A$ 
nor lcn-parents of $A$. 
\end{definition}

The Markov condition in Definition \ref{definition:LCNcondition} is:
\begin{equation}
\label{equation:LCNcondition}
X \; \indep \; \mathcal{N} \backslash \{ \{A\} \cup \mbox{lcn-de}(A) \cup \mbox{lcn-pa}(A)\} \mid \mbox{lcn-pa}(A),
\end{equation}
where we use $\indep$ here, and in the remainder of the paper, to 
mean ``is independent of''. 

We will often use the superscript $c$ to mean complement, hence
$ \mathcal{A}^c \doteq \mathcal{N} \backslash \mathcal{A}$. 

\citet{Qian2022} have derived inference algorithms (that is, they consider
the computation of conditional probabilities) that exploit such
independence relations, and they examine applications that demonstrate
the practical value of LCNs. 

It seems that a bit more discussion about the meaning of this Markov condition,
as well as its properties and consequences, would be welcome.
To do so, we find it useful to introduce a
novel concept, namely, the {\em structure} of a LCN.

\begin{figure*}
\centering
\begin{tikzpicture}
\node[draw,rectangle,very thick,dashed] (f2f3) at (1,6) {$F_2 \wedge F_3$};
\node[draw,rectangle,very thick,dashed] (f1f3) at (3.5,6) {$F_1 \wedge F_3$};
\node[draw,rectangle,very thick,dashed] (f1f2) at (6,6) {$F_1 \wedge F_2$};
\node[draw,rectangle,rounded corners,fill=yellow] (f1) at (1,4.5) {$F_1$};
\node[draw,rectangle,rounded corners,fill=yellow] (f2) at (3.5,4.5) {$F_2$};
\node[draw,rectangle,rounded corners,fill=yellow] (f3) at (6,4.5) {$F_3$};
\node[draw,rectangle,very thick,dashed] (s1s2) at (1,3) {$S_1 \vee S_2$};
\node[draw,rectangle,very thick,dashed] (s2s3) at (3.5,3) {$S_2 \vee S_3$};
\node[draw,rectangle,very thick,dashed] (s1s3) at (6,3) {$S_1 \vee S_3$};
\node[draw,rectangle,rounded corners,fill=yellow] (s1) at (1,1) {$S_1$};
\node[draw,rectangle,rounded corners,fill=yellow] (s2) at (3.5,1) {$S_2$};
\node[draw,rectangle,rounded corners,fill=yellow] (s3) at (6,1) {$S_3$};
\node[draw,rectangle,very thick,dashed] (ns1) at (2,0.25) {$\neg S_1$};
\node[draw,rectangle,very thick,dashed] (ns2) at (4.5,0.25) {$\neg S_2$};
\node[draw,rectangle,very thick,dashed] (ns3) at (7,0.25) {$\neg S_3$};
\node[draw,rectangle,rounded corners,fill=yellow] (c1) at (1,-0.5) {$C_1$};
\node[draw,rectangle,rounded corners,fill=yellow] (c2) at (3.5,-0.5) {$C_1$};
\node[draw,rectangle,rounded corners,fill=yellow] (c3) at (6,-0.5) {$C_1$};
\draw[->,>=latex,thick,dotted] (f2f3)--(f1);
\draw[->,>=latex,thick,dotted] (f1f3)--(f2);
\draw[->,>=latex,thick,dotted] (f1f2)--(f3);
\draw[->,>=latex,thick,dotted] (f1)--(f1f3);
\draw[->,>=latex,thick,dotted] (f2)--(f2f3);
\draw[->,>=latex,thick,dotted] (f2)--(f1f2);
\draw[->,>=latex,thick,dotted] (f3)--(f1f3);
\draw[->,>=latex,thick,dotted] (f3) -- (f2f3); 
\draw[->,>=latex,thick,dotted] (f1) -- (f1f2); 
\draw[->,>=latex,thick,dotted] (f1)--(s1s2);
\draw[->,>=latex,thick,dotted] (f2)--(s2s3);
\draw[->,>=latex,thick,dotted] (f3)--(s1s3);
\draw[->,>=latex,thick] (s1)--(c1); 
\draw[->,>=latex,thick] (s2)--(c2); 
\draw[->,>=latex,thick] (s3)--(c3); 
\draw[->,>=latex,thick,dotted] (s1s2) to[out=-80,in=80] (s1);
\draw[->,>=latex,thick,dotted] (s1) to[out=100,in=-100] (s1s2);
\draw[->,>=latex,thick,dotted] (s2s3) to[out=-80,in=80] (s2);
\draw[->,>=latex,thick,dotted] (s2) to[out=100,in=-100] (s2s3);
\draw[->,>=latex,thick,dotted] (s1s3) to[out=-80,in=80] (s3);
\draw[->,>=latex,thick,dotted] (s3) to[out=100,in=-100] (s1s3);
\draw[->,>=latex,thick,dotted] (s1s2) to[out=-50,in=160] (s2);
\draw[->,>=latex,thick,dotted] (s2) -- (s1s2);
\draw[->,>=latex,thick,dotted] (s2s3) to[out=-50,in=160] (s3);
\draw[->,>=latex,thick,dotted] (s3) -- (s2s3);
\draw[->,>=latex,thick,dotted] (s1s3) to[out=-150,in=10] (s1);
\draw[->,>=latex,thick,dotted] (s1) -- (s1s3);
\draw[thick,blue] (f1)--(f2);
\draw[thick,blue] (f2)--(f3);
\draw[thick,blue] (f1) to[out=-15,in=-165] (f3);
\draw[thick,blue] (s1)--(s2);
\draw[thick,blue] (s2)--(s3);
\draw[thick,blue] (s1) to[out=-15,in=-165] (s3);
\draw[->,>=latex,thick,blue] (f1) to[out=-140,in=120] (s1);
\draw[->,>=latex,thick,blue] (f1)--(s2);
\draw[->,>=latex,thick,blue] (f2) to[out=-140,in=120] (s2);
\draw[->,>=latex,thick,blue] (f2)--(s3);
\draw[->,>=latex,thick,blue] (f3) to[out=-40,in=60] (s3);
\draw[->,>=latex,thick,blue] (f3) to[out=-120,in=25] (s1);
\draw[->,>=latex,thick,dotted] (s1) -- (ns1);
\draw[->,>=latex,thick,dotted] (s2) -- (ns2);
\draw[->,>=latex,thick,dotted] (s3) -- (ns3);
\draw[->,>=latex,thick,dotted] (ns1) -- (c1);
\draw[->,>=latex,thick,dotted] (ns2) -- (c2);
\draw[->,>=latex,thick,dotted] (ns3) -- (c3);
\node at (0,2) {(a)};
\end{tikzpicture}
\hspace*{15mm}
\begin{tikzpicture}
\node[draw,rectangle,rounded corners,fill=yellow] (f1) at (1,2.5) {$F_1$};
\node[draw,rectangle,rounded corners,fill=yellow] (f2) at (3.5,2.5) {$F_2$};
\node[draw,rectangle,rounded corners,fill=yellow] (f3) at (6,2.5) {$F_3$};
\node[draw,rectangle,rounded corners,fill=yellow] (s1) at (1,1) {$S_1$};
\node[draw,rectangle,rounded corners,fill=yellow] (s2) at (3.5,1) {$S_2$};
\node[draw,rectangle,rounded corners,fill=yellow] (s3) at (6,1) {$S_3$};
\node[draw,rectangle,rounded corners,fill=yellow] (c1) at (1,0) {$C_1$};
\node[draw,rectangle,rounded corners,fill=yellow] (c2) at (3.5,0) {$C_1$};
\node[draw,rectangle,rounded corners,fill=yellow] (c3) at (6,0) {$C_1$};
\draw[->,>=latex,thick] (s1)--(c1); 
\draw[->,>=latex,thick] (s2)--(c2); 
\draw[->,>=latex,thick] (s3)--(c3); 
\draw[thick,blue] (f1)--(f2);
\draw[thick,blue] (f2)--(f3);
\draw[thick,blue] (f1) to[out=-15,in=-165] (f3);
\draw[thick,blue] (s1)--(s2);
\draw[thick,blue] (s2)--(s3);
\draw[thick,blue] (s1) to[out=-15,in=-165] (s3);
\draw[->,>=latex,thick,blue] (f1) -- (s1);
\draw[->,>=latex,thick,blue] (f1)--(s2);
\draw[->,>=latex,thick,blue] (f2) -- (s2);
\draw[->,>=latex,thick,blue] (f2)--(s3);
\draw[->,>=latex,thick,blue] (f3) -- (s3);
\draw[->,>=latex,thick,blue] (f3) -- (s1);
\node at (0,0) {(b)};
\node[draw,rectangle,rounded corners,fill=yellow] (hf123) at (1.5,-2) {$F_{1,2,3}$};
\node[draw,rectangle,rounded corners,fill=yellow] (hs123) at (1.5,-3) {$S_{1,2,3}$};
\node[draw,rectangle,rounded corners,fill=yellow] (hc1) at (0.5,-4) {$C_1$};
\node[draw,rectangle,rounded corners,fill=yellow] (hc2) at (1.5,-4) {$C_2$};
\node[draw,rectangle,rounded corners,fill=yellow] (hc3) at (2.5,-4) {$C_3$};
\draw[->,>=latex,thick] (hf123)--(hs123);
\draw[->,>=latex,thick] (hs123)--(hc1);
\draw[->,>=latex,thick] (hs123)--(hc2);
\draw[->,>=latex,thick] (hs123)--(hc3);
\node at (0,-3) {(c)};
\node[draw,rectangle,rounded corners,fill=yellow] (hf12) at (4.2,-2) {$F_{1,2}$};
\node[draw,rectangle,rounded corners,fill=yellow] (hf23) at (5.8,-2) {$F_{2,3}$};
\node[draw,rectangle,rounded corners,fill=yellow] (hs123) at (5,-3) {$S_{1,2,3}$};
\node[draw,rectangle,rounded corners,fill=yellow] (hc1) at (4,-4) {$C_1$};
\node[draw,rectangle,rounded corners,fill=yellow] (hc2) at (5,-4) {$C_2$};
\node[draw,rectangle,rounded corners,fill=yellow] (hc3) at (6,-4) {$C_3$};
\draw[thick] (hf12)--(hf23);
\draw[->,>=latex,thick] (hf12)--(hs123);
\draw[->,>=latex,thick] (hf23)--(hs123);
\draw[->,>=latex,thick] (hs123)--(hc1);
\draw[->,>=latex,thick] (hs123)--(hc2);
\draw[->,>=latex,thick] (hs123)--(hc3);
\node at (3.5,-3) {(d)};
\end{tikzpicture}
\caption{(a) The dependency graph for the LCN in Example~\ref{example:Smokers},
together with the edges in the structure of the LCN. Edges in the structure
are solid (the ones added in the process are in blue); edges
in and out of formula-nodes are dotted.
(b) The structure of the LCN, by removing the formula-nodes and
associated edges.
(c) A directed acyclic graph with the chain components of
the chain graph that represents the structure.
(d) A variant discussed in Example \ref{example:Factorization}.}
\label{figure:DependencyStructure}
\end{figure*}
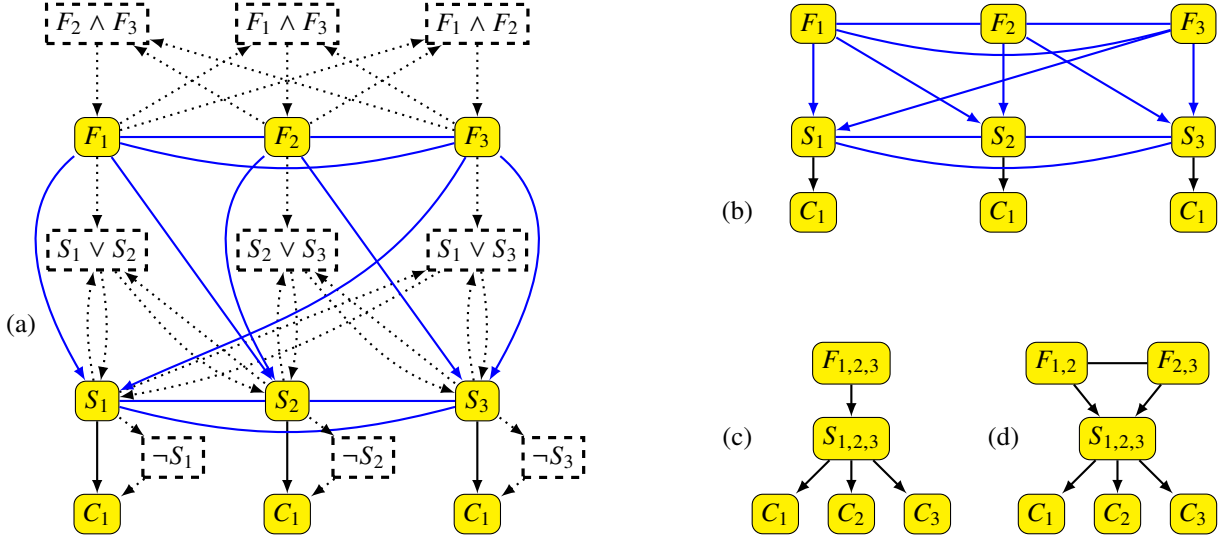

\section{The Structure of a LCN}

The dependency graph of a LCN is rather similar in spirit to
the {\em factor graph} of a Bayesian network \cite{Koller2009},
where both random variables and conditional probabilities are
explicitly represented. This is a convenient device when it comes
to message-passing inference algorithms, but perhaps it contains
too much information when one wishes to examine independence relations.

We introduce another graph to be extracted from the dependency
graph of a given LCN, that
we call the {\em structure} of the LCN, as follows:
\begin{enumerate}
\item For each formula-node
$\phi$ that appears as a 
conditioned formula in a constraint in $\mathcal{T}_U$,
place an undirected edge between any two propositions
that appear in $\phi$.
\item For each pair of formula-nodes $\varphi$ and $\phi$
that appear in a constraint, add a directed edge from
each proposition in $\varphi$ to each proposition in $\phi$.
\item If, for some pair of proposition-nodes $A$ and $B$, there is
now a pair of edges $A \leftrightarrows B$, then replace both edges
by an undirected edge. 
\item Finally, remove multiple edges
and remove the formula-nodes and all edges in and out of them.
\end{enumerate} 

\begin{example}
Figures \ref{figure:DependencyStructure}.a and 
\ref{figure:DependencyStructure}.b depict  the
  structure of the LCN in Example \ref{example:Smokers}.
\end{example}

We have:
\begin{lemma}\label{theorem:ParentBoundary}
The set of lcn-parents of a proposition $A$ in a LCN is identical 
to the boundary of $A$ with respect to the structure of the LCN.
\end{lemma}
\begin{proof} 
Consider a LCN with a dependency graph $\mathcal{D}$. 
If $B$ is a lcn-parent of $A$ with respect to $\mathcal{D}$,
then $B \rightarrow A$ or
$B \rightarrow \phi \rightarrow A$  
or
$B \rightarrow \varphi \rightarrow \phi \rightarrow A$  
or
$B \leftrightarrows \phi \rightarrow A$  
or
$B \rightarrow  \phi \leftrightarrows A$  
or
$B \leftrightarrows \phi \leftrightarrows A$  
for formula $\varphi$ and $\phi$; 
hence $B$ appears either as a parent of $A$   or a neighbor 
in the structure of the LCN.
Conversely, if $B$ is a parent or a neighbor of $A$ 
in the structure of the LCN, then one of the sequences of edges
already mentioned must be in $\mathcal{D}$, so $B$ is a
lcn-parent of $A$ in $\mathcal{D}$.
\end{proof}

The natural candidate for the concept of ``descendant'' in a structure, so as to mirror the concept of lcn-descendant, is as follows:

\begin{definition}\label{definition:QianDescendant}
If there is a directed path from $A$ to $B$ such that no intermediate 
node is a boundary node of $A$, then $B$ is a {\em strict descendant} of $A$.
\end{definition}

Using the previous definitions, we can state a local Markov condition that works for any graph but that, when applied to the structure of a LCN, has the same effect
as the original Markov condition LMC(LCN)
(Definition \ref{definition:LCNcondition}) applied to the LCN:
\begin{definition}[LMC(C-STR)]
\label{definition:LCNconditionStructures}
A node $A$ is independent, given its boundary,
of all nodes that are not $A$ itself nor strict descendants of $A$
nor boundary nodes of $A$.
\end{definition}
In symbols,
\begin{equation}
\label{equation:ConditionStructure}
X \; \indep \; \mathcal{N} \backslash \{ \{A\} \cup \mbox{sde}(A) \cup \mbox{bd}(A)\} \mid
\mbox{bd}(A),
\end{equation}
where $\mbox{sde}(A)$ denotes the set of strict descendants of $A$.

 We then have:

\begin{theorem}\label{theorem:EqualityMarkovConditions}
Given a LCN, the
Markov condition LMC(LCN) in Definition \ref{definition:LCNcondition}
is identical, with respect to the independence relations it imposes,
to the local Markov condition LMC(C-STR) 
in Definition \ref{definition:LCNconditionStructures}
applied to the structure of the LCN.
\end{theorem}
\begin{proof}
To prove that Expressions (\ref{equation:LCNcondition}) and (\ref{equation:ConditionStructure})
are equivalent, we use the fact that $\mbox{lcn-pa}(A)$ and $\mbox{bd}(A)$ are identical
by Lemma \ref{theorem:ParentBoundary} and we prove (next) that
$\mathcal{N} \backslash \{ \{A\} \cup \mbox{lcn-de}(A) \cup \mbox{lcn-pa}(A)\}$ is equal to 
$\mathcal{N} \backslash \{ \{A\} \cup \mbox{sde}(A) \cup \mbox{bd}(A)\}$.


Suppose then that 
$B \in \mathcal{N} \backslash \{ \{A\} \cup \mbox{lcn-de}(A) \cup \mbox{lcn-pa}(A)\}$ 
and assume, to obtain a contradiction, that
$B \in (\mathcal{N} \backslash \{ \{A\} \cup \mbox{sde}(A) \cup \mbox{bd}(A)\})^c$.
So, our assumption is that $B \in \{A\} \cup \mbox{sde}(A) \cup \mbox{bd}(A)$,
and the latter union can be written as 
$\{A\} \cup \mbox{bd}(A) \cup ( (\mbox{bd}(A))^c \cap \mbox{sde}(A) )$,
a union of disjoint sets. So it may be {\em either} that \\
$\bullet$ We have $B=A$,  a contradiction. \\
$\bullet$ We have $B \in \mbox{bd}(A)$. Then $B$ is either a parent or a neighbor, and in 
both cases there must be an edge from $B$ to $A$ in the dependency graph and then
$B \in \mbox{lcn-pa}(A)$, a contradiction. \\
$\bullet$ We have $B \in (\mbox{bd}(A))^c \cap \mbox{sde}(A)$. Then there
must be a directed path from $A$ to $B$ (with intermediate nodes that are
not boundary nodes) in the structure, and so there must be a corresponding directed
path from $A$ to $B$ (with intermediate nodes that are not parents) in the 
dependency graph; so $B \in \mbox{lcn-de}(A)$, a contradiction. \\
So, we always get a contradiction; hence if 
$B \in \mathcal{N} \backslash \{ \{A\} \cup \mbox{lcn-de}(A) \cup \mbox{lcn-pa}(A)\}$  
then
$B \in \mathcal{N} \backslash \{ \{A\} \cup \mbox{sde}(A) \cup \mbox{bd}(A)\}$.

Suppose now that we have a node $B$, distinct from $A$, such that 
$B \in \mathcal{N} \backslash \{ \{A\} \cup \mbox{sde}(A) \cup \mbox{bd}(A)\}$ 
and assume, to obtain a contradiction, that
$B \in ( \mathcal{N} \backslash \{ \{A\} \cup \mbox{lcn-de}(A) \cup \mbox{lcn-pa}(A)\} )^c$.
The reasoning that follows is parallel to the one in the last paragraph, but this case
has a few additional twists to take care of. 
So, our assumption is that 
$B  \in  \{A\} \cup \mbox{lcn-de}(A) \cup \mbox{lcn-pa}(A)$, and the latter union
can be written as 
$\{A\} \cup \mbox{lcn-pa}(A) \cup ( (\mbox{lcn-pa}(A))^c \cap \mbox{lcn-de}(A) )$,
again a union of disjoint sets. So it may be {\em either} that \\
$\bullet$ We have $B=A$,  a contradiction. \\
$\bullet$ We have $B \in \mbox{lcn-pa}(A)$. Then there is a directed edge from $B$ to $A$,
or a bi-directed edge between them, and $B \in \mbox{bd}(A)$, a contradiction.  \\
$\bullet$ We have $B \in (\mbox{lcn-pa}(A))^c \cap \mbox{lcn-de}(A)$. Then there
must be a path from $A$ to $B$ in the dependency graph (with intermediate nodes that are
not parents), and by construction of the
structure there must be a path from $A$ to $B$
in the structure (with intermediate nodes that are
not boundary nodes). The latter path cannot be an undirected path;
otherwise, it would have passed through a parent of $A$ in the dependency graph, 
contradicting the assumption that $B$ is a lcn-descendant. As there is a directed 
path that cannot go through a parent in the structure, $B \in \mbox{sde}(A)$, a contradiction. \\
So, we always get a contradiction; hence if 
$B \in \mathcal{N} \backslash \{ \{A\} \cup \mbox{sde}(A) \cup \mbox{bd}(A)\}$ 
then
$B \in  \mathcal{N} \backslash \{ \{A\} \cup \mbox{lcn-de}(A) \cup \mbox{lcn-pa}(A)\}$.
Thus the latter two sets are identical and the proof is finished.
\end{proof}

\section{Chain Graphs and Factorization}\label{section:NoDirectedCycles}

If the structure of a LCN is a directed acyclic graph, the LMC(LCN)
is actually the usual local Markov condition for directed
acyclic graphs as applied to Bayesian or credal networks
\cite{CozmanAI2000,Maua2020IJARthirty}.  
If instead all constraints in a LCN belong to $\mathcal{T}_U$, all
of them only referring to ``unconditional'' probabilities (that is,
$\varphi=\top$ in every constraint), 
then the structure of the LCN is an undirected graph
endowed with the usual local Markov condition for undirected graphs. 

These previous points can be generalized in a satisfying way 
{\em whenever the structure contains no directed cycle}:
\begin{theorem}\label{theorem:EqualityChainLCN}
If the structure of a LCN  is a chain graph, and probabilities
are positive, then the
Markov condition LMC(LCN) in Definition \ref{definition:LCNcondition}
is identical, with respect to the independence relations it imposes,
to the LMC(C) applied to the structure. 
\end{theorem}

Before we prove this theorem, it should be noted that sets of descendants
and strict descendants are not identical. This is easy to see in graphs
with directed cycles:  in Figure \ref{figure:GraphExamples}.b, node $B$
has descendants $\{A,C,D\}$ and strict descendants $\{A,D\}$. But even 
in chain graphs we may have differences: for instance, suppose that
in Figure \ref{figure:GraphExamples}.e we add a single directed edge 
from $D$ to a new node $E$; then
$E$ is the only descendant of $B$, but $B$ has no strict descendants.

In fact, the descendants of a node $A$ can be divided into two sets. 
First, a node $B$ is in the first set if and only if there is at least
one directed path from $A$ to $B$ that starts with a directed edge.
All of those nodes are strict descendants of $A$ when the structure
is a chain graph. 
Second, a node $C$ is in the second set if and only if all directed 
paths from $A$ to $C$ start with an undirected edge. 
Then all directed paths from $A$ to $C$ first reach a node that 
is in the boundary of $A$  
and consequently $C$ is {\em not} a strict descendant of~$A$. 
When the structure is a chain graph,
the first set is thus exactly $\mbox{sde}(A)$, and
we migh refer to the second set as $\mbox{wde}(A)$, the set
of ``{\em weak}'' descendants of $A$. 
By construction we have $\mbox{sde}(A) \cap \mbox{wde}(A)=\emptyset$.
Moreover, we conclude (see Figure~\ref{figure:DescendantsDiagram}) that
\[
(\mbox{sde}(A))^c = (\mbox{de}(A))^c \; \cup \; \mbox{wde}(A),
\]
from which we have that 
$\mathcal{N} \backslash (\{A\} \cup \mbox{sde}(A) \cup \mbox{bd}(A)) =
 \{A\}^c \cap ((\mbox{de}(A))^c \; \cup \; \mbox{wde}(A)) \cap (\mbox{bd}(A))^c$.
Note that $\left( \{A\}^c \cap \mbox{wde}(A) \cap (\mbox{bd}(A))^c \right)=\mbox{wde}(A)$
by construction, hence $\mathcal{N} \backslash (\{A\} \cup \mbox{sde}(A) \cup \mbox{bd}(A))$
is the union of two disjoint sets:
\begin{equation}
\label{equation:StrictWeak}
\left(  \{A\}^c \cap (\mbox{de}(A))^c \cap (\mbox{bd}(A))^c  \right) \;\; \cup \;\;
  \mbox{wde}(A). 
\end{equation}

\begin{figure}
    \centering
\begin{tikzpicture} 
\draw[pattern=dots] (1,1) rectangle (7,3.5);
\draw[gray,fill=gray] (4,1) rectangle (6,2.5);
\draw[fill=lightgray] (2.5,1.7) rectangle (5.5,3.3);
\draw[pattern=dots] (2.5,1.7) rectangle (5.5,3.3);
\draw[darkgray,fill=darkgray] (2,1) rectangle (4,2.5);
\draw[white,fill=white] (4,2.6) circle (0.3);
\draw[pattern=dots] (4,2.6) circle (0.3);
\node at (4,2.6) {$A$};
\node[fill=lightgray,rounded corners] at (4.8,2.1) {$\mbox{bd}(A)$}; 
\node[white] at (3,1.3) {$\mbox{sde}(A)$};
\node[white] at (5,1.3) {$\mbox{wde}(A)$};
\draw (1,1) rectangle (7,3.5);
\end{tikzpicture}
\caption{Nodes of a LCN whose structure is a chain graph. The dotted
area contains the nodes that are not descendants.
In a chain graph, the strict descendants cannot be
in the boundary.}
\label{figure:DescendantsDiagram}
\end{figure}
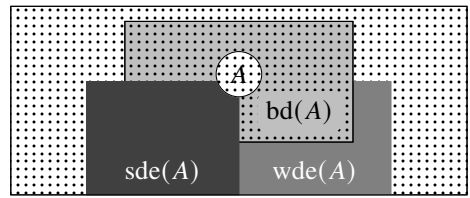

\vspace*{2ex}

\begin{proof}
Suppose we have the LMC(C-STR); so, for any node $A$, we have
Expression (\ref{equation:ConditionStructure}). Using
Expression (\ref{equation:StrictWeak}), we have
$A \indep ((\{A\}^c \cap (\mbox{de}(A))^c \cap (\mbox{bd}(A))^c) 
                                \cup \mbox{wde}(A))|\mbox{bd}(A)$;
using the Decomposition property of probabilities,\footnote{The
Decomposition property states that $X \indep Y \cup W|Z$ implies
$X \indep Y|Z$ for sets of random variables $W, X, Y, Z$ \cite{Koller2009}.}
we
obtain $A \indep \{A\}^c \cap (\mbox{de}(A))^c \cap (\mbox{bd}(A))^c|\mbox{bd}(A)$.
Thus the LMC(C) holds. 

Suppose  the LMC(C) holds; given then positivity condition,
  the GMC(C) holds \cite{Cowell99}. 
Take a node $A$ and suppose there is a node 
$B \in \mathcal{N} \backslash (\{A\} \cup \mbox{sde}(A) \cup \mbox{bd}(A))$.
We now prove that $B$ is separated from $A$ by $\mbox{bd}(A)$ in
\begin{eqnarray*}
\mathcal{G}^*_A \!\!\!\! & \doteq & \!\!\!\!\mathcal{G}^{ma}(\{A\}\!\cup\!(\{A\}^c\!\cap\!(\mbox{sde}(A))^c\!\cap\!(\mbox{bd}(A))^c)\!\cup\!\mbox{bd}(A)) \\
& = &  \mathcal{G}^{ma}((\mbox{sde}(A))^c),
\end{eqnarray*}
and consequently the GMC(C) leads to 
$A \indep (\{A\}^c \cap (\mbox{sde}(A))^c \cap (\mbox{bd}(A))^c)|\mbox{bd}(A)$ as desired.
Note  that, by construction,
nodes in $\mbox{sde}(A)$ cannot be in the ancestral set of
any node in $\mbox{wde}(A)$, so $\mbox{sde}(A)$ is not in $\mathcal{G}^*_A$
and consequently $\mathcal{G}^*_A$ is the moral graph of the graph consisting
of nodes in $(\mbox{sde}(A))^c$ and edges among them in the structure. 
Now reason as follows.
Paths that go from $A$ to a parent or neighbor are obviously blocked by $\mbox{bd}(A)$.
The only possible ``connecting'' paths from $A$ to $B$ in 
$\mathcal{G}^*_A$ 
would have to start with an undirected edge from $A$ to say $C$, an edge
added to $\mathcal{G}^*_A$ 
because $A$ and $C$ have directed edges to a common chain component. 
Now suppose $A$ points to node $D$ in this latter chain component.
Then $D$ is not in the boundary of $A$ (no directed cycles)
and the only possible way for $D$ to be in $\mathcal{G}^*_A$ 
is if there is a directed path from $D$ to a node $E$ in $\mbox{wde}(A)$, but the resulting
directed path from $A$ to $E$ would create a contradiction (as $E$ would then
not be in $\mbox{wde}(A)$). Such ``connecting'' paths from $A$ to $B$ are thus impossible
in $\mathcal{G}^*_A$.
Hence we must have separation of $A$ and $B$ by $\mbox{bd}(A)$ in $\mathcal{G}^*_A$. 
\end{proof}

The significance of the previous theorem is that, assuming 
that all probabilities are positive, the local
Markov condition for a chain graph is equivalent both to the
global Markov condition and to the factorization property of
chain graphs. This  allows us to
break down the probability distribution over all random
variables in a LCN in hopefully much smaller pieces
that require less specification effort.


\begin{example}\label{example:Factorization}
Figure \ref{figure:DependencyStructure}.b
depicts a structure that is in fact a chain graph.
We can group variables to obtain  chain components
$F_{1,2,3}$ and $S_{1,2,3}$ and
draw a directed acyclic graph with the chain
components, as in
Figure \ref{figure:DependencyStructure}.c.
The joint probability distribution factorizes as 
Expression (\ref{equation:BayesNet}):
\begin{eqnarray*}
\pr{F_{1,2,3} = f,S_{1,2,3} = s,C_1\!=\!c_1,C_2\!=\!c_2,C_3\!=\!c_3} & = &  \\
& & \hspace*{-7cm}  \pr{F_{1,2,3}=f}\pr{S_{1,2,3}=s|F_{1,2,3}=f} \nonumber \\
& & \hspace*{-6cm} \pr{C_1=c_1|S_{1,2,3}=s}\pr{C_2=c_2|S_{1,2,3}=s} \nonumber \\
& & \hspace*{-5cm} \pr{C_3=c_3|S_{1,2,3}=s}, \nonumber
\end{eqnarray*}
where $f$ is a configuration of the random variables in $F_{1,2,3}$, while
$s$ is a configuration of the random variables in $S_{1,2,3}$.
Because there are no independence relations ``inside'' the
chain components, this factorization is guaranteed even if
some probability values are equal to zero \cite{Cowell99}. 

Suppose that the three constraints $0.5 \leq \pr{F_i|F_j \wedge F_k} \leq 1$
were replaced so that we had a similar structure but
instead of a single chain component with $F_1$, $F_2$, $F_3$, suppose we had
two chain components, one with $F_1$ and $F_2$, the other with $F_2$ and $F_3$.
The chain components might be organized as in Figure \ref{figure:DependencyStructure}.d.
If all probabilities are positive, that chain graph leads
to a factorization of the joint probability distribution similar
to the previous one, but now
$\pr{F_{1,2,3}=f_1f_2f_3}=
   \mathbb{F}_1(F_{1,2}=f_1f_2) \mathbb{F}_2(F_{2,3}=f_2f_3)$,
where $\mathbb{F}_1$ and $\mathbb{F}_2$ are positive functions,
and the values of $F_1$, $F_2$ and $F_3$ are indicated by
$f_1$, $f_2$, $f_3$ respectively.
\end{example}

In the previous paragraph the assumption that probabilities are positive
is important: when some probabilities are zero, there is no guarantee
that a factorization   actually exists~\cite{Moussouris74}.
This is unfortunate  as a factorization leads to valuable computational
simplifications. 
One strategy then is to guarantee that all configurations do
have positive probability, possibly by adding language directives
that bound probabilities from below. A language command might consist
of explict bounds, say $0.001$, or even a direct guarantee of
positivity without an explicit bounding value. 
This solution may be inconvenient if we do have some hard constraints
in the domain. For instance, we may impose that $A \vee B$ (in which
case $\pr{\neg A \wedge \neg B}=0$). 
However, is is still possible to obtain a factorization
if hard constraints are imposed. Say we have a formula, for instance
$A \vee B$, that must be satisfied. We treat it as a constraint
$1 \leq \pr{A \vee B} \leq 1$ in $\mathcal{T}_U$, thus guaranteeing that
there is a clique containing its propositions/random variables.
Then we remove the impossible configurations of these random variables
(in our running example, we remove $A = B = 0$), thus reducing the
number of possible configurations for the corresponding clique.
A factorization is obtained again in the reduced space of
configurations, provided the remaining configurations  do have positive probabilities.
Finally, an entirely different strategy may be pursued:
adopt a stronger Markov condition
that guarantees factorization (and hence global independence relations)
in all circumstances.
\citet{Moussouris74} has identified one such condition, where a system
is {\em strongly Markovian} in case a Markov condition holds for the system
and suitable sub-systems. That (very!) strong condition forces zero
probabilities to be, in a sense, localized, so that probabilities 
satisfy a nice factorization property;
alas, the condition cannot be guaranteed for all graphs, and its consequences 
have not been explored in depth so far.


\section{Directed Cycles}

As noted already, existence of a factorization is a very desirable 
property for any probabilistic formalism: not only it simplifies 
calculations, but it also emphasizes modularity in modeling and ease
of understanding. 
In the previous section we have shown that LCNs whose structure is
a chain graph do have, under a positivity assumption, a well-known
factorization property. We now examine how that result  
might be extended when structures have directed cycles. 

The LMC(LCN) is, of course, a local condition that can be applied
even in the presence of directed cycles. However, local Markov conditions
may not be very satisfactory in the presence of directed cycles, as
a simple yet key example suggests:

\begin{example}\label{example:LongCycle}
Take a LCN whose dependency graph is a long cycle 
$A_1 \rightarrow A_2 \rightarrow \dots A_k \rightarrow A_1$,
for some large $k$. No $A_i$ has any non-descendant non-parent. 
And no $A_i$ has any non-strict-descendant non-parent either.
The local Markov conditions we have contemplated do not
impose {\em any} independence relation. 
\end{example}


Local conditions seem too weak when there are long cycles.
On the other hand, a global condition may work fine in those settings.
For instance, apply the GMC(C) to the graph in Example \ref{example:LongCycle};
the condition does impose non-trivial independence relations such as
$A_1 \indep A_3,\dots,A_{k-1}|A_2,A_k$ and
$A_2 \indep A_4,\dots,A_{k}|A_1,A_3$
(and more generally,
for any $A_i$ with $2<i<k-2$,
we have $A_i \indep A_1,\dots,A_{i-2},A_{i+2},A_k | A_{i-1},A_{i+1}$).

At this point it is mandatory to examine results by \citet{Spirtes95},
as he has studied local and global conditions for directed graphs,
obtaining factorization results even in the presence of directed cycles.
The local Markov condition adopted by \citet{Spirtes95} is just
the one adopted for directed graphs in Section \ref{section:Graphs}:
\begin{definition}[LMC(D)]\label{definition:SpirtesLMC}
A node $A$ is independent, given its parents,
of all nodes that are not $A$ itself nor descendants nor parents of $A$. 
\end{definition}
The global Markov condition adopted by \citet{Spirtes95} is just the
GMC(C) (Definition \ref{definition:SpirtesGMC}).
Spirtes shows that the LMC(D) is not equivalent to the GMC(C) 
for directed graphs with directed cycles.
This observation can be adapted to our setting as follows:

\begin{example}\label{example:Spirtes}
Suppose we have a LCN whose dependency graph is depicted
in Figure \ref{figure:GraphExamples}.d. 
For instance, we might have 
$0.1 \leq \pr{X|Y} \leq 0.2$ whenever $Y \rightarrow X$ is an edge
in that figure. 
Assume all configurations have positive probability.

The LMC(D) applied to this dependency graph yields only 
$A \indep C$, $B \indep C|A,D$ and $A \indep D|B,C$.

However, if we apply the GMC(C) directly
to the dependency graph, we {\em do not} get the same independence
relations: then we only
obtain  $A \indep C$ and $A \indep C|B,D$, perhaps a surprising
result (in this case, the graph $\mathcal{G}^{ma}(A,B,C,D)$ is
depicted in Figure \ref{figure:MoralGraphs}.d).
\end{example}

\citet[Lemma 3]{Spirtes95} has shown that, for a directed graph
that may have directed cycles, a positive probability distribution
over the random variables is a product of factors, one per random
variable, iff the distribution satisfies the GMC(C) for the graph.
Note that the GMC(C) is equivalent, {\em for
graphs without directed cycles}, under
a positivity assumption, to the LMC(C). 

However, there is a difficulty in applying Spirtes' result to our setting.

\begin{example}\label{example:SpirtesContinued}
Consider Example \ref{example:Spirtes}.
The structure of the LCN is the chain graph in Figure \ref{figure:GraphExamples}.e, 
and we know that the LMC(LCN) is equivalent to the LMC(C) and GMC(C) for
chain graphs. In fact, the LMC(LCN), the LMC(C), the LMC(C-STR), and the GMC(C)  
also yield only 
$A \indep C$, $B \indep C|A,D$ and $A \indep D|B,C$ when applied to the structure.
Clearly this is not the same set of independence relations imposed by
the GMC(C) applied to the dependency graph (as listed in Example \ref{example:Spirtes}).
There is a difference between undirected and bi-directed edges
when it comes to the GMC(C). 
\end{example}


The message of this example is that we cannot impose the GMC(C)
on (a suitable version of)   directed dependency graphs and hope
to keep the LMC(LCN) by \citet{Qian2022}. 
If we want the factorization induced
by the GMC(C) on (a version of) dependency graphs, we must
somehow modify the original semantics for LCNs by \citet{Qian2022}.


It is worth summarizing the discussion so far.
First, it is well-known that the LMC(C) and the GMC(C) are equivalent,
under a positivity assumption, for chain graphs (both conditions
may differ in the presence of directed cycles). 
Second, we know that the LMC(LCN) for dependency graphs is
equivalent  to 
the LMC(C-STR) with respect to the corresponding structures. 
And if the structure is a chain graph, then the LMC(C-STR) and 
the LMC(C) are equivalent when applied to the structure. 
But for general dependency graphs any
local condition seems quite weak. 
We might move to general dependency graphs by adapting
the GMC(C) to them, so as to look for a factorization result;
however, we saw that the result is not equivalent to
what we obtained by applying the GMC(C) to structures.

In the next section we examine   alternative semantics
that are based on applying the GMC(C) to structures (possibly
with directed cycles).
Before we jump into that, it is worth noticing that
there are many other relevant results in the literature besides
the ones by Spirtes. 
For instance, {\em dependency networks} \cite{Heckerman2000}
allow for directed cycles and do have a modular specification
scheme; they have only an approximate factorization, but that
may be enough in applications. Another proposal has been
advanced by \citet{Schmidt2009}, where directed cycles are
allowed and the adopted Markov condition looks only at the 
{\em Markov blanket} of nodes; it does not seem that a factorization
has been proven for that proposal, but it is attractive in its simplicity.
There are also many kinds of graphs that have been contemplated to
handle causal loops and dynamic feedback systems 
\cite{Baier2022,Hyttinen2012,Bongers2021}. 
This is indeed
a huge literature, filled with independence conditions and 
factorization properties, to which we cannot do justice in 
the available space. It is necessary to examine whether
we can bring elements of those previous efforts into
LCNs. We leave a more detailed study for 
the future.
 
\section{New Semantics for LCNs}

In this section we explore new semantics for LCNs by
applying the GMC(C) to structures.
This is motivated by the weakness of local conditions as discussed
in the previous section, and also on the fact that
a condition based on moralized graphs is the most obvious route
to factorization properties (as the Hammersley-Clifford theorem can then
be invoked under a positivity assumption~\cite{Moussouris74}). 

Here is a (new) semantics: a  LCN represents
the set of probability distributions over its nodes such that all
constraints in the LCN are satisfied, and each distribution satisfies
the GMC(C) with respect to the structure. 
Note that the GMC(C) is equivalent to the LMC(LCN) when a structure
is a chain graph, but these conditions may differ in the presence
of directed cycles. 

The path to a factorization result is then as follows. Take the mixed-structure
and, for each node $A$, build a set $\mathcal{C}_A$ with all nodes 
that belong to directed cycles starting at $A$. If there a directed 
cycle in a set $\mathcal{C}_B$ such that $B$ is in $\mathcal{C}_A$, then
merge $\mathcal{C}_A$ and $\mathcal{C}_B$ into a set $\mathcal{C}_{A,B}$; 
repeat this until 
there are no more sets to merge (this must stop, in the worst case
with a single set containing all nodes). For each set, replace all
nodes in the set by a single ``super''-node, directing all edges
in and out of nodes in the set to this super-node. The resulting
graph has no directed cycles, so the GMC(C) applied to it
results in the usual factorization over chain components of 
the resulting graph. Now each super-node is in fact a set of 
nodes that can be subject to further factorization, even though
it is an open question whether a decomposition can be obtained
with factors that are directly related to graph properties. 


To continue, we suggest that, 
instead of using structures as mere secondary objects that help
us clarify the meaning of dependency graphs,  
structures should be the primary tools in dealing with LCNs.
That is, we should translate every LCN to its structure (without 
going through the dependency graph) and then apply appropriate
Markov conditions there. Given a LCN, we can build its structure
by taking every proposition as a node and then:
\begin{enumerate}
\item For each constraint $\alpha \leq \pr{\phi|\varphi} \leq \beta$ in
$\mathcal{T}_U$, add an undirected arrow between each pair of
proposition-nodes in $\phi$.
\item For each constraint $\alpha \leq \pr{\phi|\varphi} \leq \beta$ 
add a directed edge from each proposition-node in
$\varphi$ to each proposition-node in $\phi$ (if $\varphi$ is $\top$,
there is no such edge to add).
\item Remove multiple identical edges.
\item For each pair of nodes $A$ and $B$,
if there is a bi-directed edge $A \leftrightarrows B$ between them, 
replace the two edges by a single undirected edge $A \sim B$.  
\end{enumerate}
For instance, the procedure above
goes directly from the LCN in Example \ref{example:Smokers}
to the structure in Figure \ref{figure:DependencyStructure}.b. 

When we think of structures this way, we might wish 
to differentiate the symmetric connections that appear when a pair of
propositions appear in a formula $\phi$ from the mutual influences that one
proposition is conditioned on the other and vice-versa.

An alternative semantics would be as follows. Take a LCN and build a 
{\em mixed-structure} by going through the first two steps above.
That is, create a node per proposition that appears in the LCN;
then take each constraint in $\mathcal{T}_U$ and add undirected
edges between any two propositions in $\phi$, and finally take
each constraint and add a directed edge from each proposition that
appears in $\varphi$ to each proposition that appears in the 
corresponding $\phi$. 
Figure \ref{figure:BidirectedStructures} depicts the mixed-structure
for Example \ref{example:Smokers}.

Now adopt: a  LCN represents
the set of probability distributions over its nodes such that all
constraints in the LCN are satisfied, and each distribution satisfies
the GMC(C) with respect to the mixed-structure.  

\begin{figure}[t]
\centering
\begin{tikzpicture}
\node[draw,rectangle,rounded corners,fill=yellow] (f1) at (1,2.5) {$F_1$};
\node[draw,rectangle,rounded corners,fill=yellow] (f2) at (3.5,2.5) {$F_2$};
\node[draw,rectangle,rounded corners,fill=yellow] (f3) at (6,2.5) {$F_3$};
\node[draw,rectangle,rounded corners,fill=yellow] (s1) at (1,1) {$S_1$};
\node[draw,rectangle,rounded corners,fill=yellow] (s2) at (3.5,1) {$S_2$};
\node[draw,rectangle,rounded corners,fill=yellow] (s3) at (6,1) {$S_3$};
\node[draw,rectangle,rounded corners,fill=yellow] (c1) at (1,0) {$C_1$};
\node[draw,rectangle,rounded corners,fill=yellow] (c2) at (3.5,0) {$C_1$};
\node[draw,rectangle,rounded corners,fill=yellow] (c3) at (6,0) {$C_1$};
\draw[->,>=latex,thick] (s1)--(c1); 
\draw[->,>=latex,thick] (s2)--(c2); 
\draw[->,>=latex,thick] (s3)--(c3); 
\draw[->,>=latex,thick,blue] (f1) -- (s1);
\draw[->,>=latex,thick,blue] (f1)--(s2);
\draw[->,>=latex,thick,blue] (f2) -- (s2);
\draw[->,>=latex,thick,blue] (f2)--(s3);
\draw[->,>=latex,thick,blue] (f3) -- (s3);
\draw[->,>=latex,thick,blue] (f3) -- (s1);

\draw[->,>=latex,thick,blue] (f1) to[out=10,in=170] (f2);
\draw[->,>=latex,thick,blue] (f2) to[out=-170,in=-10] (f1); 
\draw[->,>=latex,thick,blue] (f2) to[out=10,in=170] (f3);
\draw[->,>=latex,thick,blue] (f3) to[out=-170,in=-10] (f2);
\draw[->,>=latex,thick,blue] (f1) to[out=15,in=165] (f3);
\draw[->,>=latex,thick,blue] (f3) to[out=-165,in=-15] (f1);

\draw[thick,blue] (s1) -- (s2); 
\draw[thick,blue] (s2) -- (s3);  
\draw[thick,blue] (s3) to[out=-165,in=-15] (s1);

\end{tikzpicture}
\caption{The mixed-structure for the LCN in Example \ref{example:Smokers}.}
\label{figure:BidirectedStructures}
\end{figure}
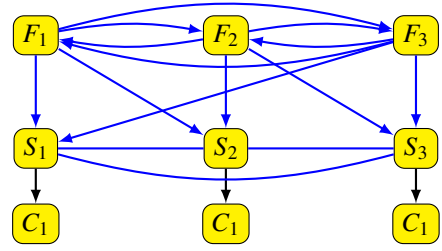


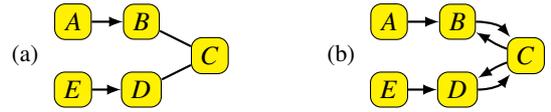
\begin{figure}[t]
\centering
\begin{tikzpicture}[scale=0.9]
\node[draw,rectangle,rounded corners,fill=yellow] (a) at (1,2) {$A$};
\node[draw,rectangle,rounded corners,fill=yellow] (b) at (2,2) {$B$};
\node[draw,rectangle,rounded corners,fill=yellow] (c) at (3,1.5) {$C$};
\node[draw,rectangle,rounded corners,fill=yellow] (e) at (1,1) {$E$};
\node[draw,rectangle,rounded corners,fill=yellow] (d) at (2,1) {$D$};
\draw[->,>=latex,thick] (a)--(b);
\draw[thick] (b)--(c);
\draw[thick] (c)--(d);
\draw[->,>=latex,thick] (e)--(d);
\node at (0.3,1.5) {\small (a)};
\end{tikzpicture}
\hspace*{1cm}
\begin{tikzpicture}[scale=0.9]
\node[draw,rectangle,rounded corners,fill=yellow] (a) at (1,2) {$A$};
\node[draw,rectangle,rounded corners,fill=yellow] (b) at (2,2) {$B$};
\node[draw,rectangle,rounded corners,fill=yellow] (c) at (3,1.5) {$C$};
\node[draw,rectangle,rounded corners,fill=yellow] (e) at (1,1) {$E$};
\node[draw,rectangle,rounded corners,fill=yellow] (d) at (2,1) {$D$};
\draw[->,>=latex,thick] (a)--(b);
\draw[->,>=latex,thick] (b) to[out=0,in=130] (c);
\draw[->,>=latex,thick] (c) to[out=160,in=-30] (b);
\draw[->,>=latex,thick] (c) to[out=-160,in=30] (d);
\draw[->,>=latex,thick] (d) to[out=0,in=-130] (c);
\draw[->,>=latex,thick] (e)--(d);
\node at (0.3,1.5) {\small (b)};
\end{tikzpicture} 
\caption{Structures and mixed-structures in Example \ref{example:MixedStructure}.}
\label{figure:MixedStructure}
\end{figure}

The next example emphasizes the differences between semantics.

\begin{example}\label{example:MixedStructure}
Suppose we have a LCN with constraints $\pr{B|A} = 0.2$,
$\pr{D|E} = 0.3$, $\pr{B \vee C} = 0.4$, $\pr{C \vee D} = 0.5$.
Both the structure and the mixed-structure of this LCN is 
depicted in Figure \ref{figure:MixedStructure}.a.
Consider another LCN with constraints 
$\pr{B|A \wedge C} = 0.2$, $\pr{C|B \wedge D} = 0.3$,
and $\pr{D|C \wedge E} = 0.4$. 
This second LCN has the same structure as the first one, but the mixed-structure
is depicted in Figure \ref{figure:MixedStructure}.b. 
The GMC(C) produces quite different sets of independence relations when applied
to these distinct mixed-structures; 
for instance, $A,B \indep D | C, E$ in the first LCN, but not necessarily in the second;
              $A,B \indep E | C, D$ in the second LCN, but not necessarily in the first.
This seems appropriate as the LCNs convey quite distinct scenarios, one related to the
symmetry of logical constraints, the other related to the links induced by directed influences.
\end{example}

We hope to pursue a  comparison between the theoretical and pragmatic aspects of these
semantics   in future work.

\section{Conclusion}\label{section:Conclusion}

In this paper we visited
many Markov conditions that can be applied, if properly adapted, to 
Logical Credal Networks \cite{Qian2022}. We reviewed existing concepts 
and introduced the notion of structure of a LCN, showing that the 
original local condition LMC(LCN) can be viewed as a local condition 
on structures. We then showed that the LMC(LCN) is equivalent to a usual 
local condition when 
the structure is a chain graph, and this leads to a factorization result.
Moreover, we introduced a new semantics 
based on structures and a global Markov condition --- a semantics 
that agrees with the original one when the structure is a chain 
graph but that offers a possible path to factorization properties.

There are many issues left for future work.
LCNs stress the connection between the syntactic form of constraints
and the semantic consequences of independence assumptions, a theme
that surfaces in many probabilistic logics. We must investigate more
carefully the alternatives when extracting independence relations from
constraints, in particular  to
differentiate   ways in which bi-directed edges are 
created.

We must also examine positivity assumptions. What is the best way to
guarantee a factorization? Should we require the user to explicitly express 
positivity 
assumptions? Should we allow for logical constraints that assign
probability zero to some configurations; if so, which kinds of
configurations, and how to make those constraints compatible with
factorization properties? 

It is also important to study a large number of Markov conditions
that can be found in the literature, both the ones connected with chain
graphs and the ones connected with causal and feedback models, that we did not 
deal with in this paper. We must verify which conditions lead
to factorization results, and which conditions are best
suited to capture the content of logical formulas, causal
influences, feedback loops. 

In a more applied perspective, we must investigate whether the ideas
behind LCNs can be used with practical specification
languages such as Probabilistic Answer Set Programming, and we must
test how various semantics for LCNs fare in realistic settings.

%


\section*{Acknowledgements}

This work was carried out at the Center for Artificial Intelligence (C4AI - USP/IBM/FAPESP), with support by the S\~ao Paulo Research Foundation (FAPESP grant 2019/07665-4) and by the IBM Corporation. The author was partially supported by CNPq grant 312180/2018-7. We acknowledge support by CAPES - Finance Code 001.



\end{document}